\documentclass[letterpaper, 10pt, journal, twoside]{IEEEtran}
\usepackage{amsmath,amsfonts,amssymb}
\usepackage{amsthm}
\usepackage{algorithmic}
\usepackage{algorithm}
\usepackage{array}
\usepackage{booktabs}
\usepackage[caption=false,font=footnotesize,labelfont=rm,textfont=rm]{subfig}
\usepackage{textcomp}
\usepackage{stfloats}
\usepackage{url}
\usepackage{verbatim}
\usepackage{graphicx}
\usepackage{subfloat}
\usepackage{hyperref}
\usepackage{bm}
\usepackage{multirow}
\usepackage{cite}
\usepackage{dsfont}
\usepackage{color}
\hypersetup{
	colorlinks=true,
	linkcolor=blue,
	filecolor=black,
	urlcolor=black,
	citecolor=blue,
}
\usepackage{balance}

\newtheorem{proposition}{Proposition}
\newtheorem{corollary}{Corollary}
\newtheorem{remark}{Remark}

\begin{document}

\title{CoDiMAD: Diffusion-Based Privileged Distillation for Communication-Free Multi-Robot Coordination}


\author{Jiyue Tao$^1$, Shunheng Xin$^1$, Tongsheng Shen$^2$, Dexin Zhao$^{2,*}$, and Feitian Zhang$^{1,*}$
	\thanks{J. Tao, S. Xin, and F. Zhang are with the Robotics and Control Laboratory, School of Advanced Manufacturing and Robotics, and the State Key Laboratory of Turbulence and Complex Systems, Peking University, Beijing, 100871, China (\href{mailto: jiyuetao@pku.edu.cn}{email: jiyuetao@pku.edu.cn}; \href{mailto: xin-shun-heng@pku.edu.cn}{email: shxin25@stu.pku.edu.cn}; \href{mailto: feitian@pku.edu.cn}{email: feitian@pku.edu.cn}).}

	\thanks{T. Shen and D. Zhao are with the National Innovation Institute of Defense Technology, Beijing 100071, China (\href{mailto: shents_bj@126.com}{email: shents\_bj@126.com}; \href{mailto: zhaodx2008@163.com}{email: zhaodx2008@163.com}).}

	\thanks{$^*$Send all correspondence to D.~Zhao and F.~Zhang.}
}

\maketitle
\markboth{}
{}

\begin{abstract}
Decentralized multi-robot coordination under partial observability remains challenging, especially when agents act solely on local sensing without communication. Privileged policy distillation provides a promising approach by transferring knowledge from a globally informed oracle to sensor-constrained students. However, in multi-agent systems, the same local observation may correspond to multiple global configurations requiring qualitatively different cooperative actions, making the conditional action distribution inherently multi-modal. Standard deterministic distillation collapses these modes to their mean, often yielding invalid or hesitant actions. To address this issue, we propose CoDiMAD, a three-stage framework that trains a privileged oracle with MAPPO, constructs an offline dataset of local-observation-oracle-action pairs, and distills the oracle into decentralized students parameterized as conditional denoising diffusion probabilistic models. By approximating the conditional oracle-action distribution through the diffusion reverse process, CoDiMAD samples decisive actions from coherent coordination modes rather than averaging across them. Theoretical analysis characterizes the mode-averaging failure of deterministic distillation and the distributional recovery property of diffusion-based distillation. Experiments on three cooperative tasks show that CoDiMAD consistently outperforms direct local MARL and deterministic distillation baselines. The source code will be made publicly available upon acceptance.
\end{abstract}


\section{Introduction}
\label{sec:introduction}

\IEEEPARstart{M}{ulti-robot} coordination underpins a growing range of applications, such as environmental monitoring\cite{zhou2024coped}, search and rescue\cite{zhang2022h2gnn}, and cooperative manipulation\cite{zhang2024heterogeneous}, where teams of autonomous robots collectively accomplish objectives beyond individual capabilities. In practical deployments, scalability and communication constraints often necessitate \emph{decentralized execution}\cite{oliehoek2016concise}, where each agent acts based on its own sensor observations without relying on a central coordinator at runtime. The widely adopted \emph{Centralized Training Decentralized Execution} (CTDE) paradigm\cite{lowe2017multi, rashid2020monotonic,tao2026arboids} addresses this by leveraging global information during training, typically through a centralized critic, while constraining actor policies to local observations at deployment.

Yet even with centralized value guidance, actors trained end-to-end from local observations usually converge slowly and settle on suboptimal coordination strategies. A primary cause is partial observability\cite{oliehoek2016concise}, as the same local observation may be consistent with multiple global configurations, each requiring a different cooperative response. This ambiguity is further amplified under \emph{communication-free execution}, where robots cannot exchange observations to reduce uncertainty. Such constraints are particularly prevalent in marine robotics, where underwater vehicles suffer from severe acoustic latency and packet loss\cite{zhou2022underwater}, and surface vessels face intermittent, range-limited radio-frequency links\cite{an2023multirobot}. A principled way to mitigate partial observability is privileged policy distillation\cite{vapnik2009new}. Instead of learning coordination directly under limited observations, one first trains an oracle policy with full state access, where the learning problem is substantially easier, and then transfers its behavior to sensor-constrained students via behavioral cloning (BC). This paradigm has achieved notable success in single-robot settings, including vision-based driving distilled from a map-equipped oracle\cite{chen2020learning} and proprioceptive locomotion distilled from a terrain-aware teacher\cite{kumar2021rma}.

Inspired by these successes, we extend privileged distillation to multi-robot coordination. However, this extension faces a fundamental challenge that is largely absent in single-robot tasks. In cooperative multi-agent systems, the oracle action for an individual robot may depend on the joint configuration of all agents and environment states beyond that robot's sensing range. Consequently, the same local observation can correspond to multiple distinct oracle actions, inducing an inherently \emph{multi-modal} conditional action distribution (Proposition~\ref{prop:info_asymmetry}). Standard deterministic policies cannot represent such distributions. Under common regression-based BC objectives such as mean squared error (MSE), they converge to the conditional mean, which may lie between distinct modes and correspond to no valid cooperative action (Corollary~\ref{cor:mode_averaging}), leading to hesitant or degenerate coordination.

To address this challenge, we propose \textbf{CoDiMAD} (\textbf{Co}operative \textbf{Di}ffusion-based \textbf{M}ulti-\textbf{A}gent \textbf{D}istillation), a three-stage framework that distills a privileged oracle into decentralized diffusion policies\cite{chi2025diffusion}. CoDiMAD first trains a MAPPO\cite{yu2022surprising} oracle with global-state access, then collects an offline dataset pairing local observations with oracle actions, and finally distills the oracle behavior into conditional denoising diffusion probabilistic models (DDPMs)\cite{ho2020denoising}. By modeling the conditional action distribution rather than a point estimate, the diffusion student avoids mode averaging and can sample decisive actions from coherent coordination modes. To overcome the high inference latency typical of diffusion models, we further integrate Denoising Diffusion Implicit Models (DDIM)\cite{song2020denoising} for accelerated, few-step sampling, thereby enabling real-time onboard deployment.

The main contributions of this work are as follows.
\begin{itemize}
	\item We propose CoDiMAD, a privileged distillation framework that resolves multi-modal ambiguity in partially observable multi-agent coordination through conditional diffusion modeling. To the best of our knowledge, this is the first work to integrate diffusion-based generative policies into multi-robot privileged distillation.

	\item We formally show that partial observability induces multi-modal action distributions under which regression-based distillation suffers from mode averaging, whereas diffusion-based generation approximates the conditional oracle action distribution.

	\item We validate CoDiMAD on three cooperative tasks, demonstrating consistent improvements over direct local MARL and deterministic distillation baselines, with analyses confirming its ability to recover multi-modal coordination behaviors.
\end{itemize}

\section{Related Work}
\label{sec:related}

\subsubsection{Multi-Agent Coordination}
The CTDE paradigm\cite{lowe2017multi,rashid2020monotonic,tao2026arboids} enables algorithms to use global information during training while executing with local observations. Many CTDE methods introduce learned communication channels\cite{sukhbaatar2016learning,foerster2016learning,das2019tarmac,jiang2018atoc} to exchange task-relevant information at execution time. However, they typically assume reliable communication, which may be unavailable in communication-free or bandwidth-limited robotics deployments. Other approaches pursue implicit coordination through shared conventions\cite{hu2021off} or heterogeneous policy architectures\cite{zhong2024harl}, but they still learn directly from partial observations and may therefore converge to suboptimal coordination strategies.

\subsubsection{Privileged Policy Distillation}
Privileged policy distillation transfers knowledge from a teacher policy with privileged information to a student policy constrained by realistic sensory inputs, and has proven effective in single-agent settings\cite{chen2020learning,kumar2021rma}. In multi-agent systems, Zhao~\emph{et al.}\cite{zhao2022ctds} introduced CTDS to distill centralized Q-values into decentralized student policies. More recently, Tang~\emph{et al.}\cite{tang2025kdmarl} investigated resource-aware distillation for resource-constrained agents, while Cho~\emph{et al.}\cite{cho2026interactive} proposed an interactive distillation framework to align student-teacher behaviors. However, these existing methods largely rely on deterministic regression objectives. Such objectives often reduce the multi-modal action distributions induced by partial observability to point estimates, thereby producing averaged, invalid coordination behaviors.

\subsubsection{Diffusion Models for Decision-Making}
Diffusion models have recently emerged as a powerful tool for decision-making, including single-agent trajectory planning\cite{janner2022diffuser} and visuomotor policy learning\cite{chi2025diffusion}, demonstrating strong capacity for modeling multi-modal behavior. In multi-agent settings, Negarmehr~\emph{et al.}\cite{negarmehr2025mimicd} trained decentralized diffusion policies via imitation learning under CTDE for implicit coordination in bimanual manipulation. However, their approach relies on expert human demonstrations, which are difficult and expensive to scale to complex, dense coordination scenarios. CoDiMAD instead uses a privileged reinforcement learning oracle for scalable supervision, retaining diffusion's multi-modal modeling advantages without human demonstrations.

\section{Problem Formulation}
\label{sec:formulation}

\begin{table}[t]
	\centering
	\caption{Observation channels for the oracle and student policies.}
	\label{tab:obs_channels}
	\begin{tabular}{lcc}
		\toprule
		Channel & Oracle & Student \\
		\midrule
		\multicolumn{3}{l}{\textit{Shared (egocentric)}} \\
		\quad LiDAR occupancy grid & \checkmark & \checkmark \\
		\quad Ego-velocity field & \checkmark & \checkmark \\
		\midrule
		\multicolumn{3}{l}{\textit{Privileged (global)}} \\
		\quad Ego-position heatmap (arena-scale) & \checkmark & $\times$ \\
		\quad Teammate position map (arena-scale) & \checkmark & $\times$ \\
		\quad Task-state map (arena-scale) & \checkmark & $\times$ \\
		\midrule
		\multicolumn{3}{l}{\textit{Local-only (student)}} \\
		\quad Teammate position map (sensing range) & $\times$ & \checkmark \\
		\quad Task-state map (sensing range) & $\times$ & \checkmark \\
		\bottomrule
	\end{tabular}
\end{table}

\subsection{Dec-POMDP Formulation}
We model multi-robot coordination as a Decentralized Partially Observable Markov Decision Process (Dec-POMDP)\cite{oliehoek2016concise}, defined by the tuple $\langle N, \mathcal{S}, \{\mathcal{A}_i\}^N_{i=1}, \{\mathcal{O}_i\}^N_{i=1}, \mathcal{P}, \mathcal{R}, \gamma \rangle$. There are $N$ agents indexed by $i\in\{1,\ldots,N\}$, sharing a global state space $\mathcal{S}$ and a team reward $\mathcal{R}:\mathcal{S}\times\mathcal{A}\to\mathbb{R}$. Each agent $i$ selects actions from a continuous individual action space $\mathcal{A}_i$, which in our robot domains corresponds to a 2D velocity command $[v_x, v_y]$. The joint action space is $\mathcal{A}=\mathcal{A}_1\times\cdots\times\mathcal{A}_N$. The state transition function $\mathcal{P}: \mathcal{S} \times \mathcal{A} \to \Delta (\mathcal{S})$ maps a state and a joint action to a probability distribution over the next states. Each agent $i$ does not observe the global state $s_t \in \mathcal{S}$ directly, but instead receives only a local observation $\mathbf{o}_i^t \in \mathcal{O}_i$ and acts without direct communication. The objective is to learn decentralized policies $\pi_i(\mathbf{a}_i^t\,|\,\mathbf{o}_i^{1:t})$ that maximize the expected cumulative team return, $J(\pi) = \mathbb{E}\left[\sum_{t=0}^{H} \gamma^t \mathcal{R}(s_t, \mathbf{a}_t)\right]$, where $H$ is the episode horizon, $\mathbf{a}_t$ denotes the joint action, and $\gamma \in [0,1)$ denotes the discount factor.

\subsection{Observation Spaces and Information Asymmetry}
We consider homogeneous robot teams with shared motion models and physical parameters, sensing configurations, and per-agent observation and action spaces. Accordingly, we apply parameter sharing across agents, while each agent conditions on its own local observations and maintains its own recurrent hidden state. To represent spatial and task-relevant information in a unified format, all local and global observations are structured as multi-channel bird's-eye-view (BEV) grid maps with a spatial resolution of $32 \times 32$, as summarized in Table~\ref{tab:obs_channels}. The privileged oracle policy processes a dual-stream input $(\mathbf{o}_i^{\mathrm{loc}}, \mathbf{o}_i^{\mathrm{glo}})$. The local stream $\mathbf{o}_i^{\mathrm{loc}} \in \mathbb{R}^{2\times32\times32}$ consists of a velocity-field map and an egocentric LiDAR occupancy map, which is constructed from uniformly spaced ray-casting measurements over $360^\circ$ with a sensing range of $r_l=60$. Each ray returns only the nearest intersection within range, thereby modeling geometric occlusion by intervening objects. The privileged global stream $\mathbf{o}_i^{\mathrm{glo}} \in \mathbb{R}^{3\times32\times32}$ contains three channels that encode arena-scale global information, including the absolute positions of the self-agent, its teammates, and the full task-state map (e.g., global target locations). In contrast, the student policy operates under realistic sensing limitations and observes only a single-stream tensor $\mathbf{o}_i \in \mathbb{R}^{4\times32\times32}$, where the three global channels are replaced by two locally-restricted channels---a cropped task-state slice and teammate positions clipped to the agent's limited sensing radius, which yields a 4-channel input combined with the two local channels from $\mathbf{o}_i^{\mathrm{loc}}$. This information asymmetry is the core source of multi-modal uncertainty: the same $\mathbf{o}_i^t$ can be consistent with multiple global configurations, each potentially requiring a different coordination action.

\section{Methodology}
\label{sec:methodology}

\begin{figure*}[t]
	\centering
	\includegraphics[width=\textwidth]{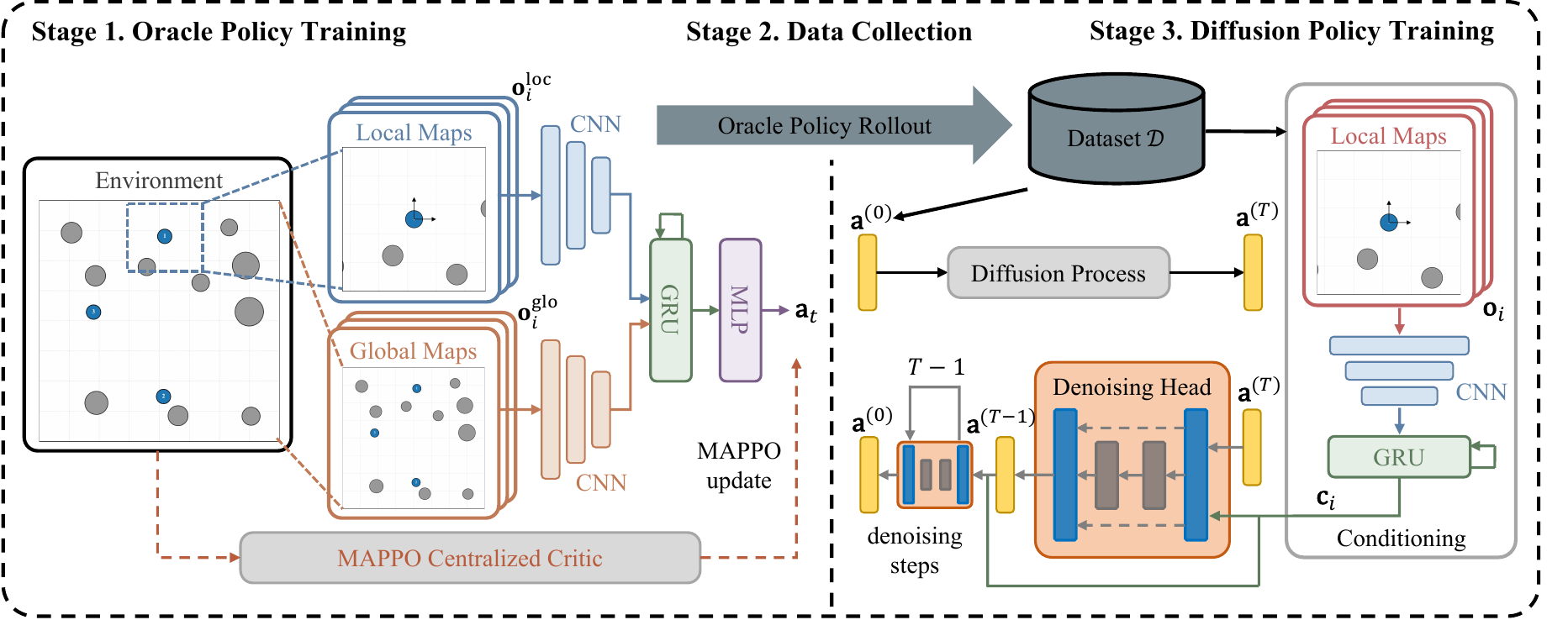}
	\caption{\textbf{Overview of the CoDiMAD framework.} (i) Stage~1 (Oracle Policy Training): The oracle actor encodes dual-stream inputs---local maps $\mathbf{o}_i^{\mathrm{loc}}$ and global maps $\mathbf{o}_i^{\mathrm{glo}}$---through parallel CNN branches, aggregates temporal features via a GRU, and produces actions $\mathbf{a}_t$ through an MLP head. A centralized MAPPO critic provides privileged value estimates during training. (ii) Stage~2 (Data Collection): The converged oracle policy is rolled out to populate an offline dataset $\mathcal{D}$ consisting of local-observation-action pairs. (iii) Stage~3 (Diffusion Policy Training): A diffusion-based student policy $\pi_\theta$ is trained on $\mathcal{D}$. A single-branch CNN--GRU encoder maps local observations $\mathbf{o}_i$ into a conditioning embedding $\mathbf{c}_i$. The denoising head $\hat{\boldsymbol{\epsilon}}_\theta$ iteratively refines Gaussian noise $\mathbf{a}^{(T)}$ into clean actions $\mathbf{a}^{(0)}$ over $T$ reverse diffusion steps conditioned on this embedding.}
	\label{fig:framework}
\end{figure*}

As illustrated in Fig.~\ref{fig:framework}, CoDiMAD comprises three stages.
Stage~1 trains a shared oracle policy $\pi^O$ via MAPPO with access to both local and privileged global observations. Stage~2 freezes $\pi^O$ and collects an offline dataset $\mathcal{D}$ of local-observation and oracle-action pairs from all agent instances and timesteps. Stage~3 distills the oracle's behavior into a shared diffusion student $\pi_\theta$ conditioned solely on local observation history. The DDPM policy samples from the conditional oracle-action distribution rather than regressing to a point estimate, avoiding the mode-averaging failure in Corollary~\ref{cor:mode_averaging}.

\subsection{Stage 1: Oracle Policy Training via MAPPO}
We train the oracle policy $\pi^O$ following the CTDE paradigm using MAPPO. In our homogeneous robot domains, the oracle actor is shared across agents and is evaluated independently for each agent. The actor receives the dual-stream observation $(\mathbf{o}_i^{\mathrm{loc}}, \mathbf{o}_i^{\mathrm{glo}})$ comprising both local and global map inputs, while the centralized critic receives the joint state for value estimation. During deployment, only the actor is retained for offline data collection, while the critic is discarded. As shown in Fig.~\ref{fig:framework}, the actor encodes $(\mathbf{o}_i^{\mathrm{loc}}, \mathbf{o}_i^{\mathrm{glo}})$ through a dual-branch convolutional encoder followed by a GRU\cite{cho2014learning} for temporal integration, and outputs a 2D velocity command $[v_x, v_y]$ via a diagonal Gaussian action head. The oracle policy is trained following the standard MAPPO configuration\cite{yu2022surprising}.

\subsection{Stage 2: Offline Trajectory Collection}
Once the oracle policy converges, it is frozen and executed \emph{deterministically} by taking the mean of its Gaussian action distribution, rather than sampling. For each rollout step, we record $(\mathbf{o}_i^t,\; \mathbf{a}_i^t,\; d_i^t)$ for all agents $i$ and timesteps $t$, where $\mathbf{o}_i^t$ is the student policy's local observation, $\mathbf{a}_i^t$ is the deterministic oracle action, and $d_i^t \in \{0,1\}$ is the episode termination indicator used for hidden-state reset. Importantly, $\mathbf{o}_i^t$ excludes the oracle's privileged global stream and contains only the information available to the student at deployment. All samples from different agent instances and timesteps are pooled into a single offline dataset $\mathcal{D}$ for training the shared student policy. To ensure high-quality supervision, we filter the collected oracle rollouts by discarding episodes that fail to complete the task or contain any collision, and retain only collision-free successful episodes in the offline dataset. Actions are then z-score normalized dimension-wise over the full dataset, with per-dimension statistics $(\boldsymbol{\mu}_a, \boldsymbol{\sigma}_a)$ stored as model buffers for denormalization at inference. 

\subsection{Stage 3: Diffusion-Based Behavioral Cloning}
\subsubsection{Architecture}
The student policy $\pi_\theta$ is a conditional denoising diffusion model comprising an observation encoder and a noise prediction network $\hat{\boldsymbol{\epsilon}}_\theta$. Unlike the oracle's dual-branch encoder, the student encoder employs a single-branch convolutional network followed by a GRU, operating exclusively on the local observation stream. At each timestep, the encoder receives the local observation $\mathbf{o}_i^t$, while the GRU maintains a recurrent hidden state that summarizes the history $\mathbf{o}_i^{1:t}$. The hidden state is reinitialized at episode boundaries and produces a conditioning embedding $\mathbf{c}_i^t$. The noise prediction network receives the noisy action $\mathbf{a}^{(k)}\in\mathbb{R}^2$, a sinusoidal embedding of the diffusion timestep $k$, and the conditioning embedding $\mathbf{c}_i^t$ as inputs. It comprises a cascade of residual MLP blocks, where the fused conditioning signal $(\mathbf{c}_i^t + \mathbf{t}_k)$ is injected into each block to preserve observation-dependent information throughout the network depth. The output is the predicted noise vector $\hat{\boldsymbol{\epsilon}}_\theta(\mathbf{a}^{(k)}, k, \mathbf{c}) \in \mathbb{R}^2$.

\subsubsection{Training Objective}
Let $\mathbf{a}^{(0)}$ denote the normalized oracle action. For notational simplicity, we omit the agent and time indices in the following diffusion equations. The forward diffusion process corrupts the clean action $\mathbf{a}^{(0)}$ by systematically adding Gaussian noise over $T=200$ timesteps according to a cosine noise schedule
\begin{equation}
    \mathbf{a}^{(k)} =
    \sqrt{\bar{\alpha}_k}\,\mathbf{a}^{(0)}
    + \sqrt{1-\bar{\alpha}_k}\,\boldsymbol{\epsilon},
    \quad
    \boldsymbol{\epsilon} \sim \mathcal{N}(\mathbf{0}, \mathbf{I}),
\end{equation}
where $k\in \{ 1,\dots,T \}$ indexes the diffusion step, and $\bar{\alpha}_k$ represents the noise schedule coefficient. To train the student policy $\pi_\theta$ to mimic the oracle's expertise under partial observability, we optimize the noise prediction network $\hat{\boldsymbol{\epsilon}}_\theta$ following the DDPM framework~\cite{ho2020denoising}. The network is trained to reconstruct the added noise $\boldsymbol{\epsilon}$ from the perturbed action $\mathbf{a}^{(k)}$ under the guidance of the conditioning context $\boldsymbol{c}$. The training loss is formulated with a per-timestep reweighting scheme
\begin{equation}
    \mathcal{L} =
    \mathbb{E}_{k,\,\boldsymbol{\epsilon},\,(\mathbf{a}^{(0)},\mathbf{o})\sim\mathcal{D}}
    \!\left[
    (1 - \bar{\alpha}_k)
    \left\|
    \boldsymbol{\epsilon}
    -
    \hat{\boldsymbol{\epsilon}}_\theta
    \big(\mathbf{a}^{(k)}, k, \mathbf{c}\big)
    \right\|^2
    \right].
    \label{eq:loss_diff}
\end{equation}
Here, $(1-\bar{\alpha}_k)$ downweights nearly clean samples and emphasizes high-noise regimes, where denoising requires stronger dependence on the observation-conditioned prior.

\subsubsection{DDIM Inference}
The full DDPM reverse process requires $T=200$ sequential denoising steps per action sample, which is costly for closed-loop robot control. Song et al.~\cite{song2020denoising} show that a non-Markovian reverse process can preserve the same training marginals as DDPM, yielding the deterministic DDIM sampler with substantially fewer denoising steps. We adopt DDIM with a uniformly sub-sampled schedule of $K=20$ steps, achieving a $10\times$ reduction in inference cost without retraining, since the same noise prediction network $\hat{\boldsymbol{\epsilon}}_\theta$ is used. Let $\{\tau_j\}_{j=0}^{K}$ denote the sub-sampled timestep sequence, where $\tau_0=0$ and $\tau_K=T$. Starting from $\mathbf{a}^{(\tau_K)}\sim\mathcal{N}(\mathbf{0},\mathbf{I})$, we iteratively compute, for $j=K,\ldots,1$,
\begin{equation}
\begin{split}
    \mathbf{a}^{(\tau_{j-1})} =\;&
    \sqrt{\bar{\alpha}_{\tau_{j-1}}}\,
    \frac{
    \mathbf{a}^{(\tau_j)}
    -
    \sqrt{1-\bar{\alpha}_{\tau_j}}\,
    \hat{\boldsymbol{\epsilon}}_\theta
    \big(\mathbf{a}^{(\tau_j)},\tau_j,\mathbf{c}\big)}
    {\sqrt{\bar{\alpha}_{\tau_j}}}
    \\
    &+
    \sqrt{1-\bar{\alpha}_{\tau_{j-1}}}\,
    \hat{\boldsymbol{\epsilon}}_\theta
    \big(\mathbf{a}^{(\tau_j)},\tau_j,\mathbf{c}\big).
\end{split}
\end{equation}
Although the DDIM update is deterministic conditioned on the initial noise, different initial noise samples can yield different action modes. The final normalized action $\mathbf{a}^{(0)}$ is denormalized using the stored statistics $(\boldsymbol{\mu}_a,\boldsymbol{\sigma}_a)$ to produce the velocity command. Fig.~\ref{fig:algorithm} summarizes the complete three-stage training procedure.

\begin{figure}[t]
	\centering
	\includegraphics[width=\linewidth]{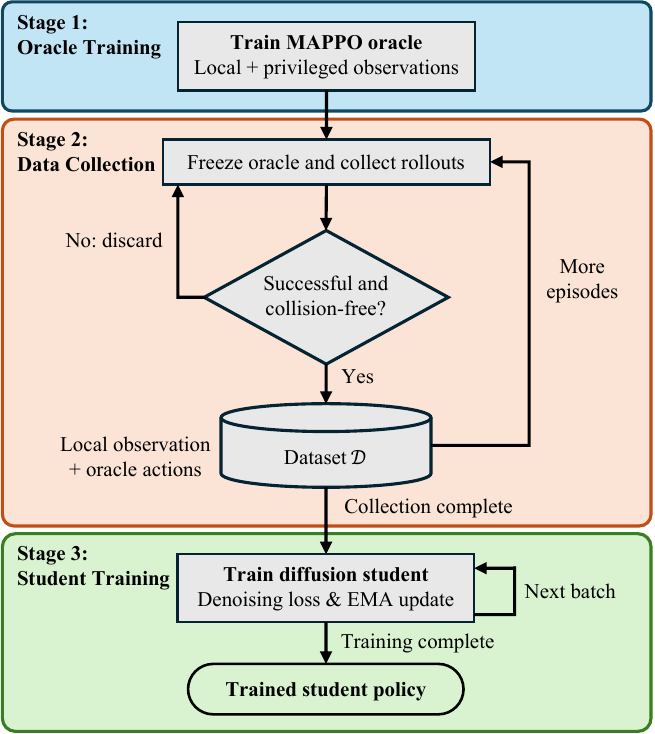}
	\caption{\textbf{Training procedure of CoDiMAD.} A privileged MAPPO oracle is first trained and frozen for rollout collection. Successful, collision-free episodes are retained to construct an offline dataset of local observations and oracle actions. A diffusion student, conditioned only on local observation history, is then trained using the denoising objective and EMA updates.}
	\label{fig:algorithm}
\end{figure}

\subsection{Theoretical Analysis}
\label{sec:theoretical} 
The key difficulty in privileged multi-agent distillation is that the oracle conditions on global information, whereas the student only observes a local projection. Consequently, a single local observation may correspond to multiple global configurations and thus to multiple valid oracle actions. This turns oracle-to-student distillation into a conditional density estimation problem rather than a simple regression problem.

\begin{proposition}[Multi-modality from partial observability]
\label{prop:info_asymmetry} 
Let $\mu_{\pi^O}(s)$ denote the deterministic oracle action at global state $s$, and let $\mathbf{o}_i=h(s)$ be the local observation available to the student. Suppose that, conditioned on a local observation $\mathbf{o}_i$, the posterior over global states decomposes into $C\ge2$ configuration classes $\{\mathcal{S}_j\}_{j=1}^C$ with weights $w_j=p(s\in\mathcal{S}_j\mid \mathbf{o}_i)>0$ and $\sum_{j=1}^C w_j=1$. Assume that the oracle actions induced by states within each class are locally distributed as $\mathcal{N}(\boldsymbol{\mu}_j,\sigma_j^2\mathbf{I})$, and that the class means are well separated:
\[
    \Delta \triangleq \min_{j\neq l}
    \|\boldsymbol{\mu}_j-\boldsymbol{\mu}_l\|
    \gg
    \sigma_{\max}\triangleq\max_j\sigma_j .
\]
Then the conditional oracle-action distribution $p_{\mathcal{D}}(\mathbf{a}\,|\,\mathbf{o}_i)$ is multi-modal, with multiple separated high-density regions.
\end{proposition}

\begin{proof}[Proof sketch]
Because the oracle policy is deterministic, conditioning on the global state
$s$ fixes the action, i.e.
$p_{\mathcal{D}}(\mathbf{a}\mid s)=\delta\!\big(\mathbf{a}-\mu_{\pi^O}(s)\big)$.
Marginalizing the hidden global state out of the posterior $p(s\mid\mathbf{o}_i)$ gives
\[
    p_{\mathcal{D}}(\mathbf{a}\mid\mathbf{o}_i)
    =\int p_{\mathcal{D}}(\mathbf{a}\mid s)\,p(s\mid\mathbf{o}_i)\,ds .
\]
Since $\{\mathcal{S}_j\}_{j=1}^{C}$ partitions the support of the posterior, we
split the integral class by class and factor out the class weights
$w_j=\int_{\mathcal{S}_j}p(s\mid\mathbf{o}_i)\,ds$:
\[
    p_{\mathcal{D}}(\mathbf{a}\mid\mathbf{o}_i)
    =\sum_{j=1}^{C} w_j\,
    p_{\mathcal{D}}(\mathbf{a}\mid s\in\mathcal{S}_j,\,\mathbf{o}_i),
\]
where each term is the action distribution induced by the states of a single
configuration class. By assumption these class-conditional distributions are approximately $\mathcal{N}(\boldsymbol{\mu}_j,\sigma_j^2\mathbf{I})$, which yields the mixture
\begin{equation}
    p_{\mathcal{D}}(\mathbf{a}\mid\mathbf{o}_i)
    \approx
    \sum_{j=1}^{C} w_j\,
    \mathcal{N}(\mathbf{a};\boldsymbol{\mu}_j,\sigma_j^2\mathbf{I}).
    \label{eq:multimodal}
\end{equation}
Under the separation condition $\Delta\gg\sigma_{\max}$, the components have
negligible pairwise overlap, so each $\boldsymbol{\mu}_j$ corresponds to a
distinct high-density region and the mixture is multi-modal.
\end{proof}

\begin{corollary}[Mode averaging under deterministic distillation]
\label{cor:mode_averaging}
Consider the two-mode special case of~\eqref{eq:multimodal},
\[
    p_{\mathcal{D}}(\mathbf{a}\mid\mathbf{o}_i)
    =
    w\,\mathcal{N}(\mathbf{a};\boldsymbol{\mu}_1,\sigma^2\mathbf{I})
    +(1-w)\,\mathcal{N}(\mathbf{a};\boldsymbol{\mu}_2,\sigma^2\mathbf{I}),
\]
where $\Delta=\|\boldsymbol{\mu}_1-\boldsymbol{\mu}_2\|$ and $w\in(0,1)$. The MSE-optimal deterministic student is the conditional mean
\begin{equation}
    f^*(\mathbf{o}_i)
    =
    w\boldsymbol{\mu}_1+(1-w)\boldsymbol{\mu}_2 .
    \label{eq:mode_avg}
\end{equation}
Let $m=\min(w,1-w)$. Then the distance from $f^*(\mathbf{o}_i)$ to its nearest mode is
\[
    \min_{j\in\{1,2\}}
    \|f^*(\mathbf{o}_i)-\boldsymbol{\mu}_j\|
    =
    m\Delta .
\]
Therefore, the mixture density at the deterministic prediction satisfies
\begin{equation}
\begin{aligned}
    p_{\mathcal{D}}\!\left(f^*(\mathbf{o}_i)\mid\mathbf{o}_i\right)
    &\leq
    (2\pi\sigma^2)^{-d_a/2}
    \exp\!\left(
    -\frac{m^2\Delta^2}{2\sigma^2}
    \right).
\end{aligned}
\label{eq:density_bound}
\end{equation}
Thus, in the well-separated regime $m\Delta\gg\sigma$, the density assigned to the deterministic prediction decays exponentially in $(m\Delta/\sigma)^2$. Consequently, the conditional mean falls in an exponentially low-density region between valid oracle actions.
\end{corollary}

\begin{remark}[Diffusion policies as conditional density models]
\label{rem:diffusion_convergence}
The denoising objective in~\eqref{eq:loss_diff} trains the model to estimate the conditional score of the noised action distribution $q_k(\mathbf{a}^{(k)}\,|\,\mathbf{c})$
\[
    \hat{\boldsymbol{\epsilon}}_\theta^*
    (\mathbf{a}^{(k)},k,\mathbf{c})
    =
    -\sqrt{1-\bar{\alpha}_k}\,
    \nabla_{\mathbf{a}^{(k)}}
    \log q_k(\mathbf{a}^{(k)}\mid\mathbf{c})
\]
in the ideal limit of infinite data and model capacity \cite{vincent2011connection,ho2020denoising}. With the exact conditional score, the ideal reverse diffusion process samples from the full distribution $p_{\mathcal{D}}(\mathbf{a}\,|\,\mathbf{o}_i)$\cite{song2021scorebased}, avoiding the mode-averaging collapse of Corollary~\ref{cor:mode_averaging}. Hence, for the multi-modal distribution in Proposition~\ref{prop:info_asymmetry}, diffusion sampling can generate actions concentrated near individual modes instead of averaging across modes. While practical implementations use finite data, finite-capacity networks, and a finite number of denoising steps, they retain the key advantage of modeling a conditional action distribution rather than a single conditional mean.
\end{remark}

\section{Experiments}
\label{sec:experiments}

\subsection{Evaluation Environments}
\label{sec:environments}
We evaluate CoDiMAD on three cooperative multi-robot tasks, as shown in Fig.~\ref{fig:environments}. All tasks are implemented in a $200\!\times\!200$ arena with $N\!=\!3$ agents by default, randomized obstacles, and 200-step episodes. The robots use a holonomic planar point-mass motion model with a first-order velocity response. Target-velocity commands are clipped componentwise, scaled by $v_{\max}=3$, and tracked using velocity smoothing ($\alpha_v = 0.2$), damping ($\rho = 0.95$), and Euclidean speed saturation. Robots have circular footprints of radius $5$ for collision checking. All lengths and
times are expressed in simulation units. The observation encoding is summarized in Table~\ref{tab:obs_channels}.

\begin{figure}[t]
	\centering
	\includegraphics[width=\linewidth]{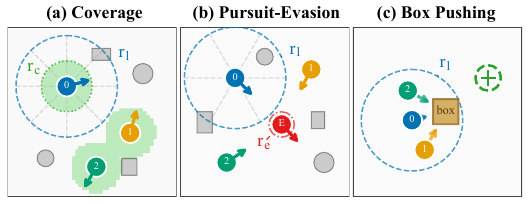}
	\caption{\textbf{Overview of the three evaluation environments.} (a)~Coverage: agents cooperatively explore the arena to maximize grid coverage. (b)~Pursuit-Evasion: pursuers cooperate to capture a faster evader. (c)~Box Pushing: agents push a heavy box to the goal region. $r_l$ denotes the LiDAR detection range, $r_c$ denotes the coverage radius, and $r_e$ denotes the capture radius.}
	\label{fig:environments}
\end{figure}

\begin{itemize}
	\item \textbf{Cooperative Coverage}: Agents explore the arena to maximize coverage of the occupancy grid. The reward encourages discovering new cells, penalizes collisions, and grants a bonus upon completing the coverage objective.
	\item \textbf{Pursuit-Evasion}: Pursuers cooperate to capture a single evader controlled by an Artificial Potential Field (APF) policy. The evader has a maximum speed twice that of the pursuers. The reward combines a capture bonus, a distance-shaping term, and a collision penalty.
	\item \textbf{Box Pushing}: Agents cooperatively push a heavy box toward a randomly placed goal zone. The reward encourages box displacement toward the goal and simultaneous multi-agent contact, while penalizing collisions.
\end{itemize}

\subsection{Experimental Setup}

\begin{table*}[t]
	\centering
	\caption{Performance comparison across three cooperative tasks.}
	\label{tab:main_results}
	\renewcommand{\arraystretch}{1.15}
	\resizebox{\linewidth}{!}{%
	\begin{tabular}{lcccccc}
		\toprule
		& \multicolumn{2}{c}{\textbf{Coverage}} & \multicolumn{2}{c}{\textbf{Pursuit-Evasion}} & \multicolumn{2}{c}{\textbf{Box Pushing}} \\
		\cmidrule(lr){2-3}\cmidrule(lr){4-5}\cmidrule(lr){6-7}
		Method & Cov.\ Rate (\%) $\uparrow$ & Collision $\downarrow$ & Cap.\ Rate (\%) $\uparrow$ & Collision $\downarrow$ & Suc.\ Rate (\%) $\uparrow$ & Collision $\downarrow$ \\

		\midrule
		\textit{MAPPO-Oracle}$^\dagger$ & \textit{97.6$\pm$0.3} & \textit{0.42$\pm$0.22} & \textit{99.1$\pm$0.6} & \textit{0.83$\pm$0.15} & \textit{98.2$\pm$1.4} & \textit{1.18$\pm$0.05} \\[3pt]
		
		MAPPO-Local & 83.7$\pm$1.0 & 0.71$\pm$0.22 & 14.2$\pm$1.0 & 1.89$\pm$0.40 & 18.0$\pm$4.3 & \textbf{1.17$\pm$0.13} \\
		BC-RNN & 80.9$\pm$0.9 & 10.37$\pm$1.15 & 90.3$\pm$1.7 & 3.41$\pm$0.48 & 6.5$\pm$2.5 & 10.32$\pm$0.67 \\
		CoDiMAD w/o RNN & 94.8$\pm$0.3 & 0.81$\pm$0.10 & 84.0$\pm$3.5 & 1.41$\pm$0.27 & 68.9$\pm$0.5 & 2.39$\pm$0.27 \\
		\textbf{CoDiMAD (Ours)} & \textbf{95.7$\pm$0.7} & \textbf{0.55$\pm$0.08} & \textbf{90.6$\pm$1.8} & \textbf{0.80$\pm$0.06} & \textbf{72.2$\pm$0.5} & 2.37$\pm$0.20 \\
		\bottomrule
		\multicolumn{7}{l}{\footnotesize $^\dagger$Non-deployable privileged policy (upper bound).}
	\end{tabular}%
	}
\end{table*}

\begin{table}[t]
    \centering
    \caption{Computational and inference efficiency comparison.}
    \label{tab:inference_efficiency}
	\renewcommand{\arraystretch}{1.2}
    \setlength{\tabcolsep}{9pt}
    \begin{tabular}{lrrrr}
    \hline
    Method & Params. & MACs & Latency & Memory \\
           & (K)     & (M)  & (ms)    & (MiB)  \\
    \hline
    MAPPO-Local & 666.0 & 0.639 & 1.14  & 19.07 \\
    BC-RNN      & 291.6 & 0.860 & 1.31  & 17.64 \\
    CoDiMAD     & 375.3 & 3.131 & 16.89 & 17.97 \\
    \hline
    \end{tabular}
\end{table}

To isolate the individual contributions of privileged distillation, temporal history encoding, and diffusion-based action generation, we compare CoDiMAD against four baseline and ablation methods, including one privileged upper bound:
\begin{itemize}
	\item \textbf{MAPPO-Oracle} (Upper Bound): The privileged policy trained with full global state via MAPPO. While non-deployable in practical scenarios, it establishes the performance ceiling.
	\item \textbf{MAPPO-Local}: MAPPO trained exclusively on local observations, representing the performance baseline of communication-free, model-free MARL.
	\item \textbf{BC-RNN}: A deterministic distillation baseline sharing the same recurrent observation encoder as CoDiMAD, but replacing the diffusion head with an MLP trained via MSE regression. This variant isolates the effect of diffusion-based action generation.
	\item \textbf{CoDiMAD w/o RNN}: An ablation that removes the GRU and conditions solely on single-frame observations, isolating the contribution of temporal history encoding.
\end{itemize}

All methods are trained and evaluated on an NVIDIA RTX 4090 GPU. The oracle policy is trained for $5$\,M steps. For the student, the noise prediction network consists of three residual MLP blocks with hidden dimension $128$. The forward diffusion process employs a cosine noise schedule ($s\!=\!0.008$, $\beta \in [10^{-4},\, 0.02]$) with $T\!=\!200$ training steps, while DDIM inference uses $K\!=\!20$ steps. The student network is optimized with AdamW ($\text{lr}\!=\!10^{-4}$) and exponential moving average (with a decay of $0.99$) for $50$ epochs with a batch size of $128$.

\subsection{Comparative Results}
\label{sec:results}

We evaluate each method across three random seeds, with each seed evaluated over 200 episodes. We report task-specific success metrics, including coverage rate, capture rate, and box-pushing success rate, together with a collision metric. The collision metric measures the average number of collision detections per episode, including inter-agent, obstacle, and boundary collisions. In Box Pushing, robot-box contacts are treated as pushing interactions and are excluded from this metric. Table~\ref{tab:main_results} summarizes the quantitative results.

\textbf{CoDiMAD vs.\ MAPPO-Oracle.}
Despite relying only on local observations and using no inter-agent communication, CoDiMAD closely approaches the privileged oracle on two of the three tasks. On Coverage, CoDiMAD achieves $95.7\%$ coverage rate compared with the oracle's $97.6\%$. On Pursuit-Evasion, CoDiMAD reaches a $90.6\%$ capture rate compared with the oracle's $99.1\%$, with similar collision frequency. On Box Pushing, CoDiMAD achieves $72.2\%$ success compared with the oracle's $98.2\%$, which is the largest gap among the three tasks. This gap reflects the tight force-coordination requirements of cooperative manipulation, where small deviations in pushing timing or angle can lead to failure. These results indicate that diffusion-based distillation transfers a substantial portion of the globally-informed cooperative behavior to a sensor-constrained decentralized student.

\textbf{Diffusion vs.\ regression (CoDiMAD vs.\ BC-RNN).}
CoDiMAD consistently outperforms BC-RNN, supporting the motivation developed in Section~\ref{sec:theoretical}. On Coverage, CoDiMAD achieves $95.7\%$ coverage rate compared with BC-RNN's $80.9\%$, corresponding to a $14.8$ percentage-point improvement. At the same time, it reduces collisions from $10.37$ to $0.55$ per episode, an approximately $19{\times}$ reduction. This large collision reduction is consistent with the mode-averaging failure of deterministic regression (cf.\ Corollary~\ref{cor:mode_averaging}). On Pursuit-Evasion, the two methods achieve comparable capture rates ($90.6\%$ vs.\ $90.3\%$), but CoDiMAD reduces collisions from $3.41$ to $0.80$, a $4.3{\times}$ improvement. This suggests that BC-RNN can often produce actions sufficient for eventual capture, yet its averaged trajectories are less coordinated and more collision-prone. On Box Pushing, the advantage of CoDiMAD becomes more pronounced: it achieves $72.2\%$ success rate compared with BC-RNN's $6.5\%$. This result suggests that the gap between CoDiMAD and deterministic regression may become larger as task difficulty increases, particularly when success requires precise and sustained multi-agent coordination.

\textbf{Distillation vs.\ direct local RL (CoDiMAD vs.\ MAPPO-Local).}
CoDiMAD achieves higher task performance than MAPPO trained end-to-end from local observations, validating the effectiveness of privileged distillation. The gap is especially large on Pursuit-Evasion, where MAPPO-Local achieves only $14.2\%$ capture rate compared with CoDiMAD's $90.6\%$. Pursuit requires coordinated behaviors such as flanking and encirclement, which are difficult to discover from local observations and sparse task rewards alone. On Box Pushing, CoDiMAD achieves a higher success rate than MAPPO-Local ($72.2\%$ versus $18.0\%$), but incurs more collisions per episode ($2.37$ versus $1.17$). Successful box pushing requires robots to assemble and interact at close range, creating opportunities for inter-agent collisions. MAPPO-Local usually fails to establish an effective coordinated pushing configuration, leaving robots more dispersed and reducing close-range interactions. Its lower collision count should therefore be interpreted together with its limited task success. However, reducing collisions during close-range cooperative pushing remains an area for improvement. On Coverage, MAPPO-Local performs comparatively better, yet CoDiMAD still outperforms it by $12.0$ percentage points with fewer collisions. 

\textbf{Temporal history (CoDiMAD vs.\ CoDiMAD w/o RNN).}
Removing the GRU leads to a moderate but consistent performance drop across all three tasks. The effect is most visible on Pursuit-Evasion, where the capture rate decreases from $90.6\%$ to $84.0\%$ and collisions increase from $0.80$ to $1.41$ per episode. On Coverage and Box Pushing, the gaps are smaller, but still indicate a positive contribution from temporal history. These results suggest that the GRU provides useful contextual information for disambiguating single-frame observations, especially in sequential coordination tasks. 

\textbf{Execution cost.} Table~\ref{tab:inference_efficiency} summarizes the computational and inference efficiency results. While CoDiMAD maintains a low peak GPU memory footprint comparable to the baselines, its iterative denoising process ($K=20$) naturally incurs higher inference latency. Specifically, CoDiMAD achieves a mean latency of $16.89$\,ms, which remains well within the typical $50$\,ms decision cycle ($20$\,Hz) of standard local motion planners~\cite{paden2016survey}. Although CoDiMAD's substantial performance gains come with the latency trade-off inherent to diffusion-based policies, these metrics confirm that it retains sufficient computational efficiency for real-time deployment.

\subsection{Analysis of Multi-Modal Behavior}
\label{sec:multimodal_analysis}

\begin{figure}[t]
	\centering
	\includegraphics[width=\linewidth]{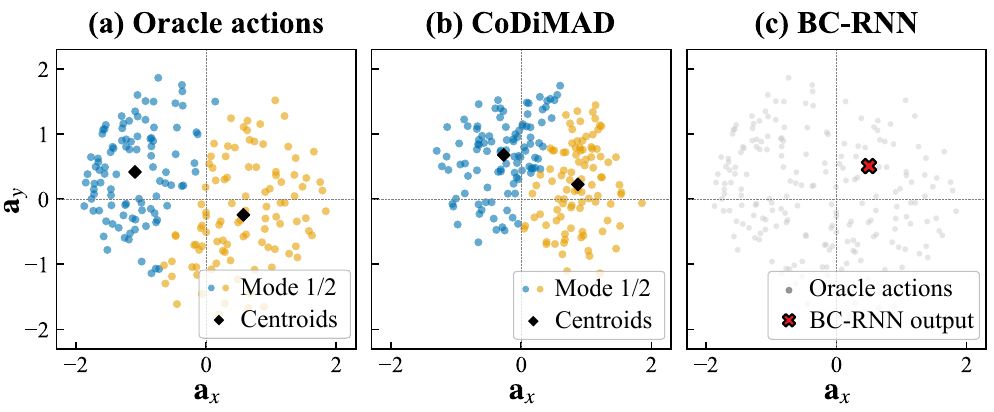}
	\caption{\textbf{Conditional action distributions under partial observability.} (a)~Oracle actions from the 200 nearest-neighbor observations in the Pursuit-Evasion dataset, revealing two distinct coordination modes. (b)~200 independent CoDiMAD samples conditioned on the same observation embedding with different initial diffusion noise, capturing both modes. (c) The deterministic BC-RNN output (red cross) plotted against the oracle distribution (gray), illustrating the mode-averaging behavior.}
	\label{fig:action_dist}
\end{figure}

\subsubsection{Action Distribution Visualization}
To visualize action multi-modality, we select a local observation $\mathbf{o}_i$ from the Pursuit-Evasion dataset whose nearest neighbors in observation space exhibit bimodal oracle actions. This indicates that similar local observations may correspond to different global configurations requiring different cooperative responses. Fig.~\ref{fig:action_dist}(a) shows the oracle actions of the 200 nearest neighbors in the 2D action space, where two clear clusters emerge. Fig.~\ref{fig:action_dist}(b) shows that CoDiMAD, sampled 200 times from the same observation embedding with independent initial diffusion noise, generates actions covering both oracle modes. In contrast, Fig.~\ref{fig:action_dist}(c) shows that BC-RNN produces a single deterministic output located between the two clusters, consistent with the mode-averaging behavior described in Corollary~\ref{cor:mode_averaging}.

\subsubsection{Trajectory Diversity}
We further examine whether action-level multi-modality leads to distinct coordination strategies over time. Starting from identical initial states, we execute five independent rollouts for both CoDiMAD and BC-RNN. As shown in Fig.~\ref{fig:trajectories}, CoDiMAD produces diverse but coherent trajectories across rollouts: agents may cover different regions in Coverage or adopt different approach paths in Pursuit-Evasion. By contrast, BC-RNN produces the same deterministic trajectory across rollouts, reflecting its lack of sampling diversity. These results show that CoDiMAD's stochastic diffusion policy preserves multi-modal behavior over extended horizons, enabling different valid coordination strategies from the same initial condition.

\begin{figure}[t]
	\centering
	\includegraphics[width=0.9\linewidth]{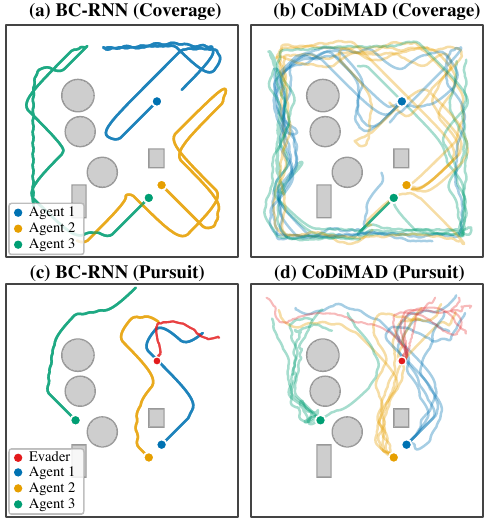}
	\caption{\textbf{Trajectory diversity from identical initial states.} Five rollouts are shown for each method from the same initial condition. (a) and (b): Coverage. (c) and (d): Pursuit-Evasion. CoDiMAD produces diverse but individually coherent trajectories, corresponding to different coordination modes. BC-RNN collapses to a single deterministic trajectory across rollouts.}
	\label{fig:trajectories}
\end{figure}

\begin{table}[t]
    \centering
    \caption{Zero-shot transfer of CoDiMAD trained with $N=3$.}
    \label{tab:population_transfer}
    \setlength{\tabcolsep}{9pt}
	\renewcommand{\arraystretch}{1.2}
    \begin{tabular}{crrrr}
        \toprule
        & \multicolumn{2}{c}{Coverage}
        & \multicolumn{2}{c}{Pursuit-Evasion} \\
        \cmidrule(lr){2-3}\cmidrule(lr){4-5}
        $N$ & Cov.\,(\%) $\uparrow$ & Collision $\downarrow$
        & Cap.\,(\%) $\uparrow$ & Collision $\downarrow$ \\
        \midrule
        3 & 95.7$\pm$0.7 & 0.55$\pm$0.08 & 90.6$\pm$1.8  & 0.80$\pm$0.06 \\
        6 & 98.8$\pm$0.8 & 2.48$\pm$0.14 & 99.0$\pm$0.3  & 3.24$\pm$0.21 \\
        9 & 99.6$\pm$0.1 & 4.63$\pm$0.22 & 100.0$\pm$0.0 & 6.86$\pm$0.35 \\
        \bottomrule
    \end{tabular}
\end{table}

\subsection{Zero-Shot Scalability}
CoDiMAD's parameter sharing and communication-free, egocentric local perception naturally decouple policy execution from the global team size $N$. Because teammates are spatially encoded only when entering a robot's local BEV, the student policy trained at a fixed team size already experiences varying local teammate counts and can be directly transferred to larger teams without retraining. To evaluate this capability, we test the fixed policies (trained with $N=3$) across $N\in\{3,6,9\}$, while keeping all other environment settings unchanged. As shown in Table~\ref{tab:population_transfer}, CoDiMAD remains highly effective at larger scales, reaching $99.6\%$ coverage and $100.0\%$ capture at $N=9$. Note that deploying more robots in a fixed-size arena naturally facilitates task completion through greater spatial parallelism, but simultaneously increases the risk of team-level collisions due to severe local crowding. Overall, these results confirm CoDiMAD's strong zero-shot scalability across team sizes, while highlighting collision avoidance in high-density regimes as an avenue for future improvement.

\section{Conclusion}
\label{sec:conclusion}

We presented CoDiMAD, a three-stage learning framework for communication-free multi-robot coordination. Experiments on three multi-agent tasks show that CoDiMAD consistently outperforms both direct local MARL and deterministic distillation baselines, while approaching the performance of the privileged oracle under communication-free decentralized execution. Action-distribution and trajectory visualizations further support our theoretical analysis: local observations can correspond to multi-modal oracle action distributions, CoDiMAD recovers distinct modes through diffusion sampling, and deterministic distillation methods tend to produce intermediate mode-averaged actions. Several limitations suggest directions for future work. The present experiments are limited to homogeneous teams. An extension to heterogeneous teams could retain the CoDiMAD framework through robot-specific MAPPO oracles, per-robot datasets from joint rollouts, and diffusion students adapted to each robot's observation and action spaces. This extension is nontrivial, however, since heterogeneous teams pose harder credit assignment problems, and is left for future work. Moreover, our simulations adopt a simplified holonomic motion model to focus on high-level coordination. Deployment on physical robots will require accounting for platform-specific dynamics and actuation constraints, integrating suitable low-level tracking controllers for sim-to-real transfer, and validating the resulting system experimentally.

\balance
\bibliographystyle{IEEEtran}
\bibliography{ref}

\end{document}